\documentclass{article}

\usepackage{microtype}
\usepackage{graphicx}
\usepackage{subcaption}
\usepackage{booktabs} 

\RequirePackage[natbibapa]{apacite}
\usepackage[nohyperref, accepted]{icml2020}

\usepackage{mathtools}
\usepackage{amsmath,amssymb}
\usepackage{dsfont}
\usepackage{amsthm}

\newtheorem{proposition}{Proposition}

\DeclareMathOperator*{\argmin}{argmin}
\DeclareMathOperator*{\argmax}{argmax}

\DeclareMathOperator{\real}{\mathbb{R}}

\DeclarePairedDelimiter\norm{\lVert}{\rVert}
\usepackage{algorithm} 
\usepackage{algorithmic}
\DeclareMathOperator{\E}{\mathbb{E}}

\DeclareMathOperator{\defeq}{\dot{=}}

\icmltitlerunning{Understanding the Pathologies of Approximate Policy Evaluation when Combined with Greedification}

\begin{document}
\setlength{\textfloatsep}{10pt}

\twocolumn[
\icmltitle{Understanding the Pathologies of Approximate Policy Evaluation\\
when Combined with Greedification in Reinforcement Learning}



\icmlsetsymbol{equal}{*}

\begin{icmlauthorlist}
\icmlauthor{Kenny Young}{alb}
\icmlauthor{Richard S. Sutton}{alb}
\end{icmlauthorlist}

\icmlaffiliation{alb}{Department of Computing Science, University of Alberta, Edmonton, Canada}

\icmlcorrespondingauthor{Kenny Young}{kjyoung@ualberta.ca}

\icmlkeywords{Machine Learning, ICML}

\vskip 0.3in
]



\printAffiliationsAndNotice{}  
\begin{abstract}
Despite empirical success, the theory of reinforcement learning (RL) with value function approximation remains fundamentally incomplete. Prior work has identified a variety of pathological behaviours that arise in RL algorithms that combine approximate on-policy evaluation and greedification. One prominent example is policy oscillation, wherein an algorithm may cycle indefinitely between policies, rather than converging to a fixed point. What is not well understood however is the quality of the policies in the region of oscillation. In this paper we present simple examples illustrating that in addition to policy oscillation and multiple fixed points--- the same basic issue can lead to convergence to the worst possible policy for a given approximation. Such behaviours can arise when algorithms optimize evaluation accuracy weighted by the distribution of states that occur under the current policy, but greedify based on the value of states which are rare or nonexistent under this distribution. This means the values used for greedification are unreliable and can steer the policy in undesirable directions. Our observation that this can lead to the worst possible policy shows that in a general sense such algorithms are unreliable. The existence of such examples helps to narrow the kind of theoretical guarantees that are possible and the kind of algorithmic ideas that are likely to be helpful. We demonstrate analytically and experimentally that such pathological behaviours can impact a wide range of RL and dynamic programming algorithms; such behaviours can arise both with and without bootstrapping, and with linear function approximation as well as with more complex parameterized functions like neural networks.
\end{abstract}
\section{Introduction}
Despite significant empirical success, value-function based algorithms for reinforcement learning (RL) control remain poorly understood theoretically. When states are represented exactly, Q-learning and Sarsa are known to converge to the optimal action-value function under reasonable conditions~\citep{tsitsiklis1994asynchronous, jaakkola1994convergence, singh2000convergence}. Unfortunately, if the action-value or state-value functions are represented by function approximation, even the relatively simple case of linear function approximation, then the picture becomes less clear and pathological behaviours can result.

\citet{gordon1996chattering} first observed that Sarsa with an $\epsilon$-greedy policy can oscillate indefinitely, failing to converge. \citet{de2000existence}, generalized this to the dynamic programming (DP) setting, by showing that approximate value iteration (AVI) can also oscillate indefinitely and may not possess fixed points. \citet{bertsekas2011approximate} further highlights that policy oscillation is a general property of a wide class of DP algorithms which utilize on-policy evaluation combined with greedification (see also~\citet{bertsekas2010pathologies}). 

What is not well understood however is to what extent such policy oscillation is a problem. For example, \citet{gordon2001reinforcement} demonstrates that though SARSA(0) can oscillate between policies, the weights nonetheless converge to a bounded region, but gives little insight into the quality of policies in this region. \citet{bertsekas2011approximate} states that ``the full ramifications of policy oscillation in practice are
not fully understood at present, but it is clear that they give serious reason for concern''. In this report we demonstrate that the same basic mechanism that leads to policy oscillations in a wide range of RL and DP algorithms can also cause convergence to the worst possible policy. The basic problem is that while the evaluation accuracy is always optimized for the distribution of states that arise under the current policy, greedification may be based on states which are nonexistent or rare under this distribution. Thus, the values used for greedification are unreliable. In addition to policy oscillation we show that even in cases where the policy does converge reliably, the final policy can be arbitrarily suboptimal. This can occur even when the optimal policy is representable in the class of greedy policies, and the on-policy evaluation error of the optimal policy can be made zero.




Our result on convergence to a suboptimal policy is related to a result by \citet{pendrith1998analysis}. They show that although the optimal representable policy is always a fixed point of Monte Carlo (MC) policy-gradient, the same is not true when the policy-gradient is estimated based on a value function optimized by TD($\lambda$). We highlight that this result is a consequence a general issue with alternating between evaluating and greedifying under function approximation and can occur even if the policy is updated using a finite look-ahead on a value function trained on full MC returns. Thus, the issue is not exclusive to evaluation methods which bootstrap, like TD($\lambda$) .

The work of~\citet{lu2018non} is also closely related. They highlight the inability of Q-learning and AVI to converge to the optimal representable policy due to what they refer to as \textit{delusional bias}. Delusional bias refers to bootstrapping values based on sequences of actions which are not mutually realizable under a greedy policy in the function approximation class. We can reframe this definition by saying that optimizing approximate action-values under the current policy can modify the approximated values of other actions such that, after greedifying, the learned action-values are no longer optimized for the new policy. Thus, delusional bias is closely related to the issues we explore here. We take a somewhat different perspective on the problem, which enables us to highlight how similar problems can arises even in the absence of bootstrapping.


\section{Background}
RL refers to the problem faced by a goal directed agent interacting with an initially unknown environment. We formalize this interaction as a Markov decision process (MDP). An MDP $\mathcal{M}$ is defined by a tuple $\mathcal{M}=(\mathcal{S},\mathcal{A},p)$. At each time $t$ an agent observes a state $S_t\in\mathcal{S}$ and based on this information selects an action $A_t\in\mathcal{A}$. Based on this action and the current state, the environment then transitions to a new state $S_{t+1}\in\mathcal{S}$, and outputs a reward $R_{t+1}\in\real$ according to the dynamics function $p(s^\prime,r|s,a)=P(S_{t+1}=s^\prime,R_{t+1}=r|S_t=s,A_t=a)$.

The agent's behaviour is specified by a policy $\pi(a|s)$, which is a distribution over $a\in\mathcal{A}$ for each $s\in\mathcal{S}$. We focus on the episodic setting, where we assume that for every possible policy the agent environment interaction eventually reaches a terminal state from which no more reward is possible. Our main conclusions can also be generalized to the cases of discounted return and average reward. The termination time $T$ is the random time at which the terminal state is reached. 

The sum of rewards obtained from some time $t$ until $T$ is called the return $G_t=\sum\limits_{k=t}^{T-1}R_{k+1}$. The agent's goal is to learn, from experience, a policy which maximizes the total expected return $G_0$. We define the state-value function under a particular policy $\pi$ as $v_\pi(s)=\E_{\pi}[G_t|S_t=s]$, that is the expected return from time $t$ given the agent starts in state $s$ and follows the policy $\pi$.


We use $\mu_\pi(s)$ to represent the visitation frequency of a state $s$ under policy $\pi$. We refer to $\mu_\pi(s)$ as the on-policy distribution.  In the continuing case, $\mu_\pi(s)$ would be the steady-state distribution. In the episodic case we consider, we take it to be the expected fraction of total time-steps spent in the state across all episodes, that is:
\begin{equation*}
    \mu_\pi(s)=\left.\E_\pi\left[\sum\limits_{t=0}^{T-1}\mathds{1}(S_t=s)\right]\middle/\E_\pi[T]\right.,
\end{equation*}
where $T$ is the random time at which the terminal state is reached and expectations are taken over episodes. In both the episodic and continuing case $\mu_\pi(s)$ represents the distribution of states visited by an online RL agent running a fixed policy $\pi$. As such, it is also the distribution implicitly used in the objective of most standard on-policy RL algorithms.

We will also discuss the DP setting. DP is similar to RL except that we assume the agent has access to the dynamics function $p$. In DP the goal is to compute a good policy directly from $p$ rather than learning from experience.

\section{The Issue with Greedification based on Approximate Evaluation}\label{counterexamples}
There is a basic issue with algorithms that alternate between on-policy evaluation with function approximation and greedification of the policy based on the resulting evaluation. The issue is that when an approximate value function is optimized to be accurate with respect to a particular on-policy distribution, the resulting evaluation of states which are rarely visited under the current policy can be essentially arbitrary depending on the function class. Yet, when we greedify the policy, we frequently rely on the evaluation of such states. This issue is present even in the absence of bootstrapping and affects a wide class of RL and DP algorithms that use approximate value functions for policy improvement. To illustrate how fundamental this issue is, we first describe how it occurs in a DP algorithm, without stochasticity, or bootstrapping. We refer to this algorithm as approximate policy iteration (API), its pseudocode is provided in Algorithm \ref{state_API}. API learns an approximate state-value function $\hat{v}(\cdot;\theta): \mathcal{S}\rightarrow\real$, which should be seen as a learned approximation to $v_\pi(s)$, as an intermediate step to learning a good policy. API alternates between optimizing $\hat{v}(\cdot;\theta)$ for the current policy by minimizing $\sum\limits_{s}\mu_\pi(s)\left(\hat{v}(s;\theta)-v_\pi(s)\right)^2$, and updating the policy to be greedy\footnote{Assume ties are broken in some canonical manner.} with respect to a one-step look-ahead on the current value function approximation. This seemingly simple algorithm displays a variety of pathological behaviours depending on the way the function approximator generalizes.
\begin{algorithm}[t]
\caption{Approximate Policy Iteration}
\label{state_API}
\begin{algorithmic}[1]
\FUNCTION{API($\pi_0,\hat{v}$)}
\STATE$\pi\leftarrow \pi_0$
\WHILE{Not Converged}
\STATE $\theta\leftarrow\argmin\limits_{\hat{v}(\cdot;\theta^\prime):\theta^\prime\in\real^n}\sum\limits_{s}\mu_\pi(s)\left(\hat{v}(s;\theta^\prime)-v_\pi(s)\right)^2$ \label{evaluation_step}
\STATE $\pi\leftarrow \text{Greedy}(\hat{v},\theta)$
\ENDWHILE
\ENDFUNCTION
\FUNCTION{Greedy($\hat{v},\theta$)}
\FOR{$s\in\mathcal{S}$}
\STATE $\hat{q}(s,a)\defeq\E_\pi[
R_{t+1}+\hat{v}(S_{t+1};\theta)|S_t=s,A_t=a]$
\STATE $\pi(s,a)\leftarrow\begin{cases}
1 &\text{for }a=\argmax\limits_{a^\prime}\hat{q}(s,a)\\
0 &\text{otherwise}
\end{cases}$
\ENDFOR
\STATE \textbf{return $\pi$}
\ENDFUNCTION
\end{algorithmic}
\end{algorithm}

Our analysis will use the idea of fixed points of API, and later AVI. By fixed point, we mean an approximate value function $\hat{v}(\cdot;\theta)$, such that $\hat{v}(s;\theta)$ remains unchanged for all $s$ after applying the greedificiation and evaluation steps of the algorithm.

This report focuses specifically on on-policy evaluation. More precisely, when the evaluation is optimized by minimizing a sum of per-state errors weighted by the on-policy distribution of the current policy. The issues we highlight can perhaps be mitigated by focusing on optimizing the evaluation with respect to a distribution which is fixed throughout training (see for example the work of~\citet{maei2010toward}). However, such methods are limiting as the achievable performance can be limited by the fixed behaviour policy used. Optimizing evaluation with respect to the on-policy distribution allows the behaviour policy to be incrementally updated as an agent learns from experience. For this reason, on-policy evaluation is a crucially important case which is worth understanding better. 

It is also possible to derive performance guarantees if we assume the function approximator is such that, for any policy, the value function approximation error for policy-evaluation is uniformly bounded over \textit{all} states (see chapter 6.2 of~\cite{bertsekas1996neuro}). However, as noted by~\citet{munos2003error}, most practical algorithms do not allow us to bound the maximum norm over all states. \citet{munos2003error} also identifies certain conditions under which meaningful error bounds are possible in the more realistic case where we assume a bound on the error under some weighted quadratic norm. Here, we highlight how pathological behaviours, including convergence to the worst possible policy, can arise even when on-policy evaluation error (the quadratic norm weighted by the on-policy distribution) can be made zero for each deterministic policy.

We begin by presenting three intuitive counterexamples using state-aggregation for function approximation. Note that state-aggregation is a subclass of linear function approximation. In each example there are two possible actions $\mathcal{A}=\{l,r\}$ which only impact the state transition in the initial state, where they take the agent along the left or right path respectively. The three examples differ only in the order in which rewards are presented along the two possible paths. The differing placement of rewards leads to three different behaviours of API by modifying how optimizing the evaluation accuracy under each of two possible deterministic policies affects the evaluation of the other. We will use $\pi_l$ and $\pi_r$ to refer to the two possible deterministic policies based on which of the two actions is selected in the initial state. The three counterexamples are illustrated in Figure \ref{state_counter_examples}. We will discuss each counterexample in detail to see how it leads to a certain pathological behaviour. 

\begin{figure*}[t]
\centering
\begin{subfigure}[t]{0.25\textwidth}
\includegraphics[width=\textwidth]{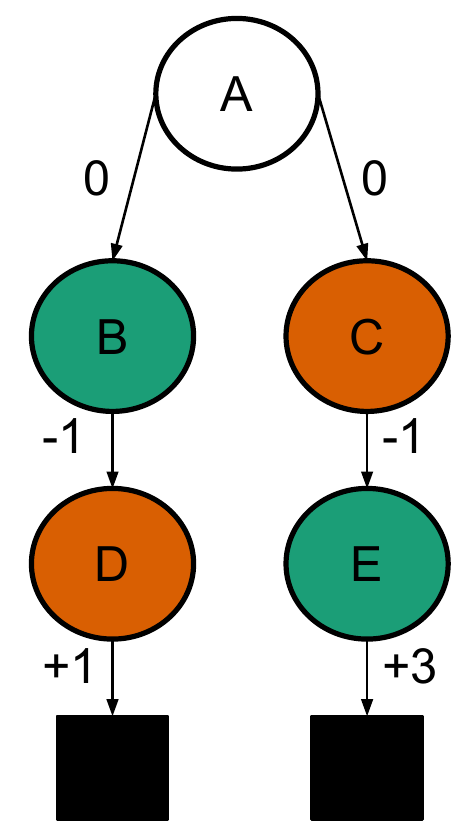}
\caption{Oscillating example: API has no fixed point.}
\label{SCE_oscillatory}
\end{subfigure}
\hfill
\begin{subfigure}[t]{0.25\textwidth}
\includegraphics[width=\textwidth]{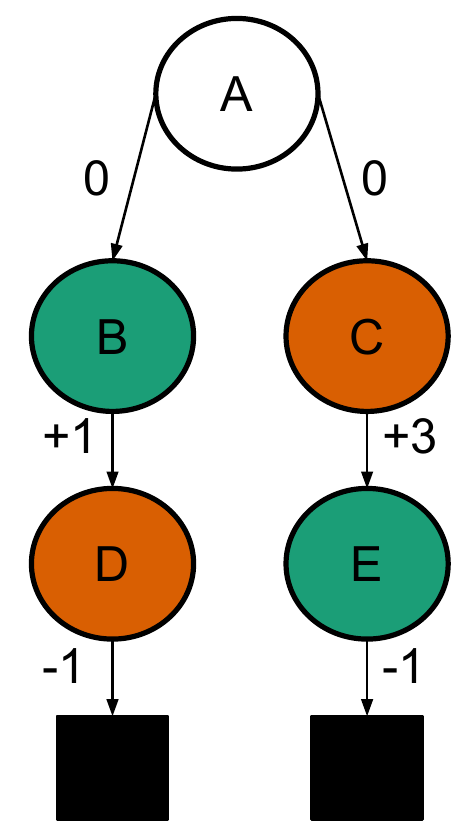}
\caption{Multiple example: API has multiple fixed points.}
\label{SCE_multi}
\end{subfigure}
\hfill
\begin{subfigure}[t]{0.25\textwidth}
\includegraphics[width=\textwidth]{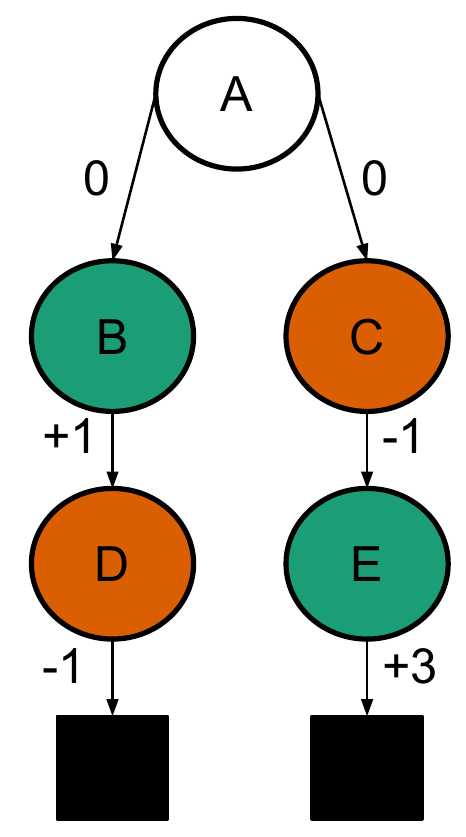}
\caption{Worst-case example: API converges to the inferior policy.}\label{SCE_bad}
\end{subfigure}
\caption{Counterexamples which lead to pathological behaviours in approximate policy iteration. Circles are states, arrows are deterministic transitions associated with particular actions. Numbers next to arrows are rewards. States of the same color are aggregated under the function approximator. In each example the right action in \textsf{A} is always optimal with a value of $V_{\pi_r}(A)=2$ compared to $V_{\pi_l}(A)=0$. The examples differ only in the order the rewards are presented along each path, yet exhibit three different pathologies of API.}
\label{state_counter_examples}
\end{figure*}

\subsection*{Oscillating Counterexample}
In the first counterexample, illustrated in Figure \ref{SCE_oscillatory} , API has no fixed point. $\pi_l$ looks superior under the optimal approximate state-value function for $\pi_r$ and visa-versa. This leads to API oscillating between the two policies.
To see this, first consider the approximate state-value function obtained by the evaluation step of API (line 4) when the policy is set to $\pi_l$. Since $\mu_{\pi_l}(\textsf{C})=\mu_{\pi_l}(\textsf{E})=0$, only the value of \textsf{B} and \textsf{D} will impact the evaluation error. The true values are $v_\pi(\textsf{B})=0$ and $v_\pi(\textsf{D})=1$. The optimal approximate state-value function will match these values exactly. Due to state aliasing, the optimal $\theta$ will correspond to $\hat{v}(\textsf{C};\theta)=1$ and $\hat{v}(\textsf{E};\theta)=0$. Now when we greedify the policy at $\textsf{A}$ we see a value of $1$ for the right action and $0$ for the left action. Thus the policy switches to $\pi_r$ in the greedification step (line 5). Optimizing the state-value function for $\pi_r$ then leads to $\hat{v}(\textsf{B};\theta)=3$ and $\hat{v}(\textsf{C};\theta)=2$. The next greedification step will switch back to $\pi_l$ and so on. 

When the policy is initialized to an arbitrary stochastic policy, greedification will take it to one of the two deterministic policies, which one will depend on the relative probability of the left and right actions. Thereafter, it will oscillate as described above. Towards making this more precise, let $\rho=\pi(r|\textsf{A})$, $\hat{v}(\textsf{A};\theta)=\theta_0$, $\hat{v}(\textsf{B};\theta)=\hat{v}(\textsf{E};\theta)=\theta_1$ and $\hat{v}(\textsf{C};\theta)=\hat{v}(\textsf{D};\theta)=\theta_2$, where $\theta_i$ is the $i_{th}$ element of the parameter vector $\theta$. For compactness define $\theta^{\star}_{\pi}=\argmin\limits_{\theta^{\prime}}\sum\limits_{s}\mu_\pi(s)\left(\hat{v}(s;\theta^{\prime})-v_\pi(s)\right)^2$. It is straightforward to show that:
\begin{equation}\label{optimal_theta}
    \begin{multlined}[b]\theta^{\star}_{\pi}
    =[\rho v_\pi(\textsf{C})+(1-\rho)v_\pi(\textsf{B}),\rho v_\pi(\textsf{E})+\\(1-\rho)v_\pi(\textsf{B}),\rho v_\pi(\textsf{C})+(1-\rho)v_\pi(\textsf{D})]
    \end{multlined}.
\end{equation}
Equation~\ref{optimal_theta} will also hold for the other two counterexamples to be outlined below. The true values of states \textsf{B}, \textsf{C}, \textsf{D}, and \textsf{E} are policy independent, replacing each with their appropriate numerical value in Equation~\ref{optimal_theta} gives us:
\begin{equation*}
\theta^{\star}_{\pi}=\left[2\rho,3\rho,1+\rho\right].
\end{equation*}
From here we can see that starting from a stochastic initial policy, evaluation followed by greedification will produce $\pi_l$ whenever $\rho=\pi(r|\textsf{A})>\frac{1}{2}$ and $\pi_r$ when $\rho< \frac{1}{2}$.

\subsection*{Multiple Fixed Point Counterexample}
In the second counterexample, illustrated in Figure \ref{SCE_multi}, API has multiple fixed points. $\pi_l$ looks superior under it's own optimal approximate state-value function and $\pi_r$ looks superior under it's own. Thus, both value functions will be fixed points. Which fixed point API converges to depends only on the initial policy.

When the state-value function is optimized for $\pi_l$ we have optimal values of $\hat{v}(\textsf{B};\theta)=0$ and $\hat{v}(\textsf{C};\theta)=-1$. Since $\hat{v}(\textsf{B};\theta)>\hat{v}(\textsf{C};\theta)$, $\pi_l$ is the greedy policy and the greedification step of API will leave it unchanged. On the other hand, when the evaluation is optimized for $\pi_r$ we have $\hat{v}(\textsf{B};\theta)=-1$ and $\hat{v}(\textsf{C};\theta)=2$, thus again the greedification step will leave the policy unchanged.  Substituting state values into Equation~\ref{optimal_theta} gives:
\begin{equation*}
\theta^{\star}_{\pi}=\left[2\rho,-\rho,3\rho-1\right].
\end{equation*}
Hence, from an arbitrary policy $\pi$, evaluation followed by greedification will produce $\pi_l$ whenever $\rho\leq\frac{1}{4}$ and $\pi_r$ otherwise. So, when the policy is initialized arbitrarily, the first greedification step will take it to one of the two deterministic policies, and it will remain there forever.

\subsection*{Worst-Case Counterexample}
In the final counterexample, illustrated in Figure \ref{SCE_bad}, API has a single fixed point, but it corresponds to the inferior of the two deterministic policies. $\pi_l$ looks superior under it's own optimal approximate state-value functions as well as that of $\pi_r$. However, $\pi_r$ is in fact the optimal policy. Thus, API will converge to the worse of the two possible deterministic policies regardless of initialization. More precisely:
\begin{equation*}
\theta^{\star}_{\pi}=\left[2\rho,3\rho,3\rho-1\right].
\end{equation*}
So from an arbitrary stochastic policy $\pi$, evaluation followed by greedification will produce $\pi_l$ whenever $3\rho\geq3\rho-1$. This is always satisfied, so API will indeed converge to the suboptimal policy from any initial policy in a single round of evaluation and greedification. This occurs despite the fact that the optimal policy is representable as a greedy policy in the class of approximate value functions, and the on-policy evaluation error of the optimal policy can be made zero.

\section{How General are these Behaviours?}
To highlight that the pathologies in section~\ref{counterexamples} are not unique to bootstrapping we introduced them within the API algorithm. In this section we explain how these examples lead to the same behaviours when bootstrapping is used, and in the RL setting.

We define the Bellman operator $B_\pi$ by $B_{\pi}v(s)=\E_\pi[
R_{t+1}+v(S_{t+1})|S_t=s]$ for every function $v:\mathcal{S}\rightarrow\real$. We similarly define the Bellman optimality operator $B^\star$ by $B^\star v(s)=\max\limits_a \E_\pi[R_{t+1}+v(S_{t+1})|S_t=s,A_t=a]$. For any policy $\pi$, $v_\pi(s)$ is the unique solution to the Bellman equation $v_{\pi}=B_\pi v_{\pi}$.

Consider the AVI algorithm in Algorithm~\ref{AVI}, which updates the approximate value function to minimize the error relative to $B^\star \hat{v}(s;\theta)$. Proposition~\ref{API_to_AVI} shows that AVI will behave similarly to API on each of our counterexamples.

\begin{proposition}\label{API_to_AVI}
For a given function approximator $\hat{v}$ and a fixed MDP $\mathcal{M}$, let it be the case that for any deterministic policy $\pi$ there exists a unique parameter setting $\theta^{*}_\pi$ such that $\sum\limits_s\mu_\pi(s)\left(\hat{v}(s;\theta^{*}_\pi)-v_\pi(s)\right)^2=0$. Then:

(i) Every fixed point of API is also a fixed point of AVI.

(ii) If $\hat{v}(s;\theta)=\theta\cdot x(s)$ is a linear function of some features $x(s)\in\real^n$ of the state $s$, then every fixed point of AVI is also a fixed point of API
\end{proposition}


\begin{proof}
(i) Any fixed point $\hat{v}(\cdot;\tilde{\theta})$ of API must be a minimizer of $\sum\limits_{s}\mu_\pi(s)\left(\hat{v}(s;\theta)-v_\pi(s)\right)^2$ under the deterministic policy $\pi=\text{Greedy}(\hat{v},\tilde{\theta})$. By assumption this minimum error must be zero. The only way this can be true is if $\mu_\pi(s)=0$ for all $s$ such that $\hat{v}(s;\tilde{\theta})\neq v_\pi(s)$. Hence:
\begin{align*}
    0&=\sum\limits_{s}\mu_\pi(s)\left(v_\pi(s)-B_\pi v_\pi(s)\right)^2\\
    &=\sum\limits_{s}\mu_\pi(s)\left(v_\pi(s)-B^\star v_\pi(s)\right)^2\\
    &=\sum\limits_{s}\mu_\pi(s)(v_\pi(s)-B^\star \hat{v}(s;\tilde{\theta}))^2,
\end{align*}
these steps follow respectively from the Bellman equation, the fact that $\pi$ must be the greedy with respect to $\hat{v}(s;\tilde{\theta})$ at the API fixed point, and the fact that $\mu_\pi(s)=0$ for every $s$ such that $\hat{v}(s;\tilde{\theta})\neq v_\pi(s)$. This means $\hat{v}(\cdot;\tilde{\theta})$ is a fixed point of the evaluation step of AVI, since the error cannot be reduced below zero. The greedification step will leave $\pi$ unchanged as the policy is already greedy. So $\hat{v}(\cdot;\tilde{\theta})$ is also a fixed point of AVI.

(ii) Any fixed point $\hat{v}(\cdot;\tilde{\theta})$ of AVI must be a minimizer of $\sum\limits_{s}\mu_\pi(s)(\hat{v}(\cdot;\theta)-B_{\pi} \hat{v}(s;\tilde{\theta}))^2$ under the deterministic policy $\pi=\text{Greedy}(\hat{v},\tilde{\theta})$. We have used the fact that $\pi$ must be greedy at a fixed point to replace $B^{\star}$ with $B_{\pi}$. 

Consider a new MDP $\mathcal{M}^\prime$ identical to $\mathcal{M}$ except with state space $\mathcal{S}^\prime$ consisting of those states with $\mu(s)\neq0$, and all transitions to states with $\mu(s)=0$ replaced with transitions to the terminal state. Notice that $v_\pi(s)$ for all $s\in\mathcal{S}^\prime$ is the same for $\mathcal{M}^\prime$ as it was for $\mathcal{M}$. Now consider the operator $\Omega_\pi$ defined by $\Omega_\pi v=\argmin\limits_{\hat{v}(\cdot;\theta^\prime):\theta^\prime\in\real^n}\sum\limits_{s\in\mathcal{S}^\prime}\mu_\pi(s)(\hat{v}(s;\theta^\prime)-B_{\pi}v)^2$ representing the evaluation step of AVI for which we know $\hat{v}(\cdot;\tilde{\theta})$ is a fixed point. For linear $\hat{v}(\cdot;\theta)$, $\Omega_\pi$ is a contraction with respect to the norm $\norm{v}=\sum\limits_{s\in\mathcal{S}^\prime}\mu_\pi(s)v(s)^2$, with a unique fixed point (see for example~\citet{tsitsiklis1997analysis}).

Next, note that $\hat{v}(\cdot;\theta^{*}_\pi)$ is a fixed point of $\Omega_\pi$. To see this note that, by assumption, $\hat{v}(s;\theta^{*}_\pi)=v_\pi(s)$ for all $s\in\mathcal{S}^\prime$, thus $B_{\pi}\hat{v}(s;\theta^{*}_\pi)=\hat{v}(s;\theta^{*}_\pi)$. Hence, we can conclude that $\hat{v}(\cdot;\tilde{\theta})=\hat{v}(\cdot;\theta^{*}_\pi)$ is the unique fixed point of $\Omega_\pi$. We know $\hat{v}(\cdot;\theta^{*}_\pi)$ minimizes the evaluation step of API. This and the fact that the policy is greedy, suffices to show that $\hat{v}(\cdot;\tilde{\theta})$ is a fixed point of API as well as AVI.
\end{proof}
\begin{algorithm}[t]
\caption{Approximate Value Iteration}
\label{AVI}
\begin{algorithmic}[1]
\FUNCTION{AVI($\pi_0,\hat{v},\theta_0$)}
\STATE $\pi,\theta\leftarrow \pi_0, \theta_0$
\WHILE{Not Converged}
\STATE $\theta\leftarrow\argmin\limits_{\hat{v}(\cdot;\theta^\prime):\theta^\prime\in\real^n}\sum\limits_{s}\mu_\pi(s)(\hat{v}(s;\theta^\prime)-B^\star \hat{v}(s;\theta))^2$
\STATE $\pi\leftarrow \text{Greedy}(\hat{v},\theta)$
\ENDWHILE
\ENDFUNCTION
\FUNCTION{Greedy($\hat{v},\theta$)}
\STATE $\hat{q}(s,a)\defeq\E_\pi[
R_{t+1}+\hat{v}(S_{t+1},\theta)|S_t=s,A_t=a]$
\FOR{$s\in\mathcal{S}$}
\STATE $\pi(s,a)\leftarrow\begin{cases}
1 &\text{for }a=\argmax\limits_{a^\prime}\hat{q}(s,a)\\
0 &\text{otherwise}
\end{cases}$
\ENDFOR
\STATE \textbf{return $\pi$}
\ENDFUNCTION
\end{algorithmic}
\end{algorithm}



Each of the counterexamples in Figure \ref{state_counter_examples} obey the assumption of Proposition \ref{API_to_AVI}, with a linear $\hat{v}$. Thus we can conclude that, for each counterexample, the fixed points of AVI are the same as those of API. So AVI has no fixed points on the oscillating counterexample, two fixed points on the multiple fixed point example, and one fixed point corresponding to the suboptimal policy on the worst-case counterexample. Proposition \ref{API_to_AVI} does not preclude AVI possessing additional cycles that are not possible for API or visa-versa. However, for our particular examples, one can indeed show that AVI will oscillate on the oscillating counterexample and converge to a fixed point on the other two examples.

One might hope that these issues would be less pronounced when moving from the full evaluation and greedification steps of API and AVI to the more incremental RL setting. One might suspect that if the policy and value function are more gradually adapted to each other, we won't see such pathological dynamics. However, we will demonstrate that, for the most part, the general issue illustrated in the counterexamples of Figure~\ref{counterexamples} still occur for standard RL algorithms that make use of approximate evaluation. 

 It's worth noting that MC policy-gradient (MC-PG), also known as REINFORCE~\citep{williams1992simple}, does not display the pathologies we highlight here, though~\citet{kakade2002approximately} have demonstrated that related issues with the on-policy distribution can lead to slow convergence even in the MC case. MC-PG makes updates according to an unbiased estimate of the gradient of a well defined objective, the expected return. MC-PG can have multiple local minima when using function approximation, but the optimal representable policy will always be the global minima, and therefore a fixed point. Assuming an appropriately annealed step-size, MC-PG cannot experience cycles or a repulsive optimal representable policy. However, there are many reasons to prefer value based methods, variance reduction being the most obvious, hence it's worth understanding the underlying issues.

To understand why more incremental RL algorithms based on approximate evaluation and greedification still lead to similar pathological behaviour, consider an actor-critic (AC) algorithm which concurrently trains a linear state-value function from MC returns, and a stochastic softmax policy which represents each state independently and estimates the policy-gradient using a one-step look-ahead on the value function. The value function in this case is continuously updated to reduce a stochastic approximation of the same evaluation error as the API algorithm. Meanwhile the policy is stochastically updated to be closer to the greedy policy of the current value function. Despite updating the policy and value more gradually than API, this procedure still maintains the same expected direction of the value and policy update at each point in parameter space. Thus, it's easy to see how the resulting dynamics are qualitatively similar to that of API on the counterexamples in Figure \ref{state_counter_examples}.

One notable behavioural difference arises when using algorithms with a continuous policy class, like the AC algorithm described above. In particular, it has been shown that the use of a continuous policy class suffices to guarantee the existence of fixed points, but not the uniqueness. Thus, on the oscillating example, algorithms which use a continuous policy class will tend to converge to some intermediate stochastic policy rather than oscillating between policies. See the work of~\citet{de2000existence} for an analysis of this point under a continuous policy variant of AVI and the work of~\citet{perkins2002existence} for an analysis of Q-learning and Sarsa with continuous policy classes. The proof in each case relies on Brouwer's Fixed Point 
Theorem~\citep{brouwer1911abbildung}, which states that any continuous function from a compact convex set onto itself must possess a fixed point. Continuous policy classes nonetheless behave similarly on the multiple fixed point and worst-case counterexamples.

\section{State-aggregation Experiments}
Instead of analytically assessing the performance of a particular class of RL algorithms on these problems, we illustrate the generality of these issues by running three RL algorithms on the proposed counterexamples and empirically demonstrate that similar behaviours result. The algorithms tested are Q-learning, AC with a MC critic, and AC with a TD($0$) critic.  The results are illustrated in Figure \ref{grid_plot}.
\begin{figure}[t]
\centering
\includegraphics[width=\columnwidth]{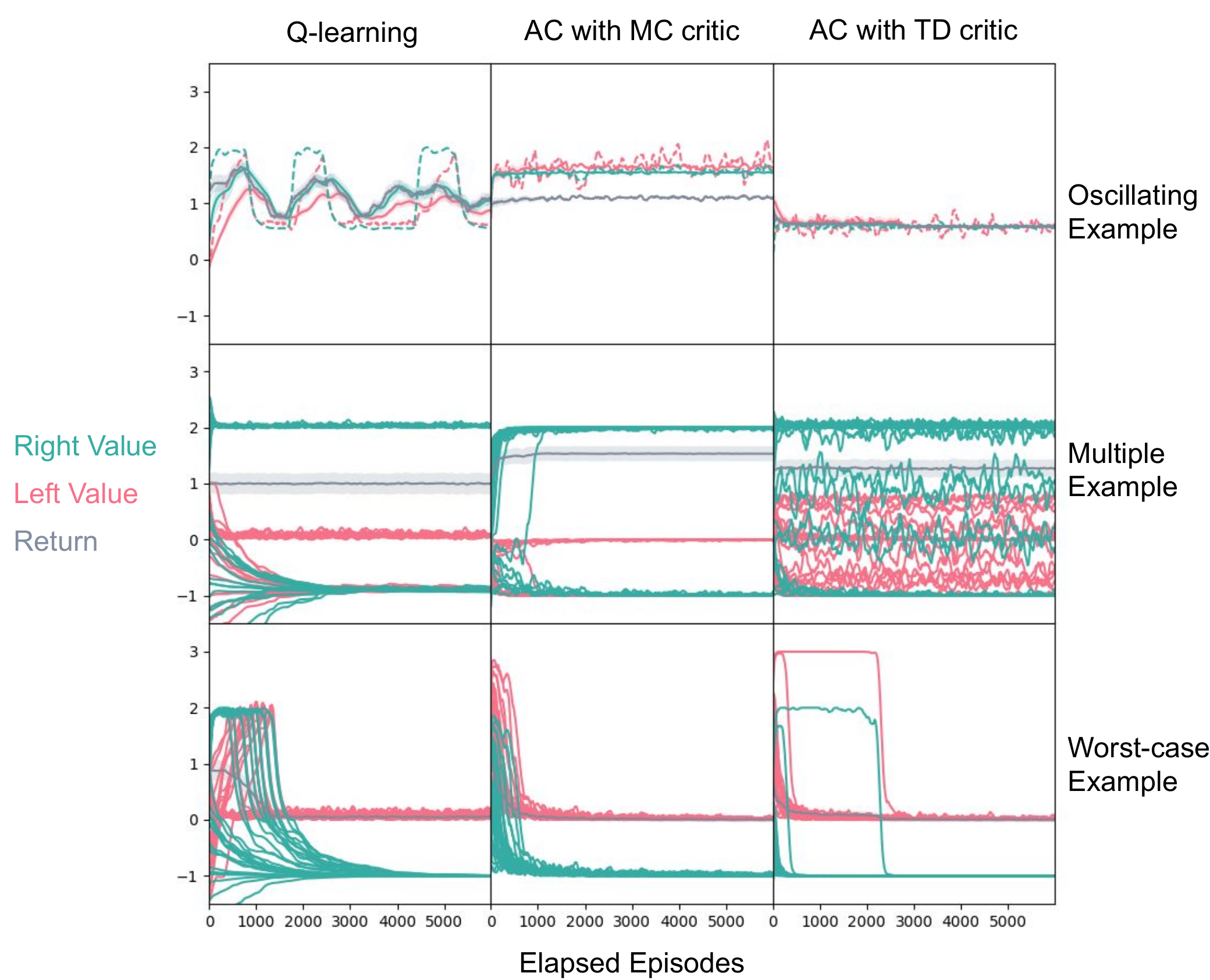}
\caption{Pathological behaviours due to approximate evaluation and greedification arise for a variety of RL algorithms. Columns show results for different algorithms, rows show results on different counterexamples. Each plot displays the result of 30 runs. Average return is plotted in grey with error bars displaying standard error. For Q-learning the action-value approximation for the left action is shown in green, while the action-value approximation of the right action is shown in red. For the two AC algorithms green and red instead indicate the approximate state-value of \textsf{B} and \textsf{C} respectively. For the oscillating counterexamples the solid lines represent the mean with error bars displaying standard error. The dotted line shows a single run to clearly illustrate the oscillating behaviour of Q-learning. For the other two counterexamples we simply display curves for each of the 30 runs.}
\label{grid_plot}
\end{figure}

In these experiments, we formulate each counterexample as having two actions available in each state, but the selected action only impacts the transition in state \textsf{A}. We represent aggregated states by running each algorithm with a tabular representation where aggregated states are treated as the same. Hence for Q-learning, we learn 3 sets of 2 action-values for state \textsf{A}, for state \textsf{B} and \textsf{E} and for state \textsf{C} and \textsf{D}. For AC we learn 3 state-values along with 3 softmax-paramaterized policies over 2 actions. A fixed, reasonably low step-size of $0.05$ is used in each case. For Q-learning we use $\epsilon$-greedy exploration with $\epsilon=0.05$. All parameters are randomly initialized from a unit normal distribution.

The results illustrate how different algorithms interact with the three counterexamples. Q-learning shows policy and value oscillation on the oscillating counterexample, while both AC algorithms converge to some stochastic policy where the value estimate of each state is approximately equal. All algorithms show two distinct fixed points on the multiple fixed point counterexample, though for AC with a TD critic the bootstrapping seems to lead to somewhat more chaotic behaviour, where some runs appear to converge to neither fixed point. All three algorithms consistently converge to the inferior policy in the worst-case counterexample. Interestingly, Q-learning initially tends toward the superior policy, before suddenly dropping off as the evaluation is optimized. 

\section{Neural Network Experiments}
We've shown how the interaction of approximate policy evaluation and greedification can lead to several pathological behaviours with linear function approximation. In this section we investigate how similar behaviours can arise with neural network function approximation. It might not be clear that with a universal function approximator like a neural network, the issues highlighted here would still be problematic. Given a network with large enough capacity, one might hope we will be able to learn the values of every state to sufficient accuracy that we can approach a good policy despite using approximate value functions for policy improvement. However, the case of a network with large enough capacity is not the only case of interest. As RL is applied to solve increasingly sophisticated problems, the worlds our agents act in will likely be much larger than the agents' representational capacity, though this is not the case in many of the domains that are commonly used for evaluation of deep RL algorithms. To give an example: ATARI 2600 games typically occupy 2-4kB~\citep{bellemare2013arcade}, while state of the art agents can use more than one million parameters~\citep{espeholt2018impala}, or more than 4000kB assuming single precision floats are used. 

When the agent lacks sufficient capacity to represent its whole world, function approximation becomes a real issue, and it is worth seriously thinking about how algorithms learn to apply limited function approximation resources. The basic issue we highlight in this report is not limited to simple function approximators like state-aggregation or linear function approximation. If the function class is not rich enough to represent all state-values accurately, then there may be a trade-off between evaluation accuracy under one on-policy distribution and another. Even in the overparameterized regime, similar effects could impact the quality of solutions obtained by RL algorithms, since the evaluation is never fully optimized in practice. In these cases, how the algorithm behaves will depend on how the function approximator generalizes.

The state representation in Figure \ref{NN_counter_examples} illustrates how similar issues can impact function approximation techniques like neural networks. This representation involves no explicit state aliasing, every state has an independent set of representations. However, the topology of the input space will force an underparameterized single hidden layer neural network with a sigmoid activation to face a trade-off similar to the counterexamples in Figure \ref{state_counter_examples}.
\begin{figure}[t]
\centering
\begin{subfigure}[t]{0.8\columnwidth}
\includegraphics[width=\textwidth]{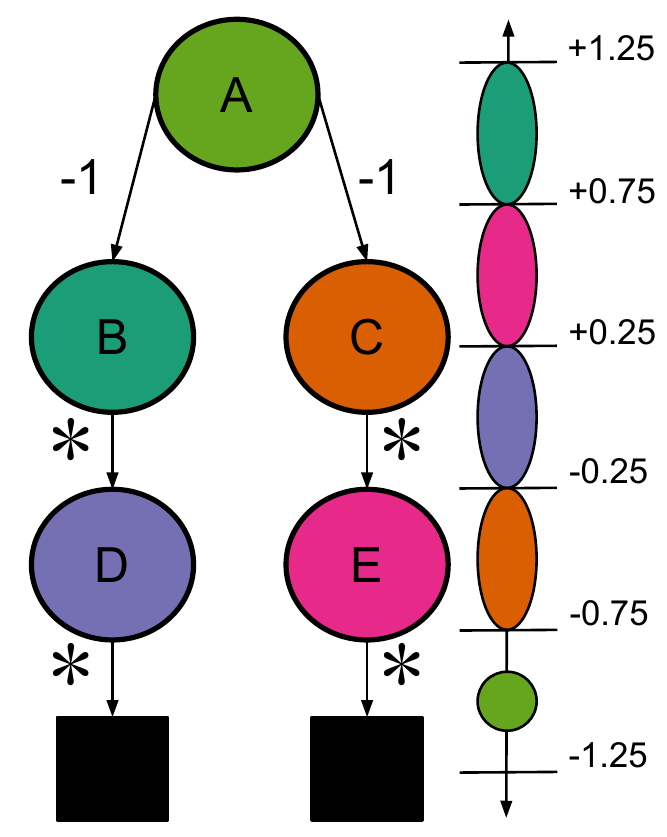}
\end{subfigure}
\caption{A state representation that illustrate how pathological behaviours due to approximated evaluation and greedification can arise even without state aliasing in an underparameterized neural network. The input representation is one dimensional with the representation of each state drawn uniformly at random on each visit, from the region marked in the corresponding color on the line on the right. In addition to changing the representation from state aliasing, each of the above examples adds an additional $-1$ reward after the first state. This was done to make the value of the start state distinct from the state that follows under a deterministic policy. This forces the function approximator to represent them separately to minimize the evaluation error. The other rewards marked with $*$ are identical to those in the associated counterexample in Figure~\ref{counterexamples}.}
\label{NN_counter_examples}
\end{figure}

Consider applying a single hidden layer neural network with linear threshold activations with the state representation in Figure ~\ref{NN_counter_examples}. If the network has just 2 hidden units it is underparameterized in the sense that it is unable to represent the value of each state independently. Such a network is limited to separating the input space into 3 intervals, each of which can have an independent value. This is, however, enough to exactly evaluate the three states along either one of the left and right paths. To exactly evaluate the states along the right path, the network must separate states \textsf{A}, \textsf{C} and \textsf{E}. Due to the topology of the representation, this means it must alias \textsf{E} and \textsf{B}. Similarly, to exactly evaluate the states along the left path it must alias \textsf{D} and \textsf{C}. Thus, the topology of the input space combined with the limitations of the function approximator leads to a situation similar to the examples given in Figure \ref{state_counter_examples}. 

Unlike the linear threshold function, a differentiable sigmoid activation can interpolate between states. To mitigate any advantage gained by doing so, we draw the representation of states from an interval, rather than a single deterministic point. To minimize evaluation error along either path the sigmoid hidden units must be close to flat within each of 2 colored intervals of Figure~\ref{NN_counter_examples}. Thus their behaviour should approximate that of linear thresholds when evaluation error is minimized. State \textsf{A} is still represented deterministically so that it's evaluation can be learned unambiguously to simplify interpretation of the results.

We trained Deep Q-network (DQN) agents on analogues to each of the problems in Figure~\ref{counterexamples} with the representation illustrated in Figure~\ref{NN_counter_examples}. We experimented with an overparameterized network with 4 sigmoid units to establish that strong performance was possible, after which we ran an underparameterized network with 2 sigmoid units to test whether we would observe similar behaviour to that illustrated in Figure~\ref{grid_plot} for the state-aggregation case.

In these experiments, we use experience replay with a batch size of 32. We use RMSProp~\citep{RMSPROP} for optimization. The step-size was selected from the set $\{0.0025\cdot2^{i}:i\in\{0,1,..,5\}\}$ to maximize performance after 500,000 episodes with the overparameterized network in the worst-case example, and the same step-size was then used for the other problems. We use $\epsilon$-greedy exploration with $\epsilon=0.1$. We also add a small amount of L2-regularization (a constant of $0.0001$) to keep the network parameters from growing inordinately large. All parameters are randomly initialized from a unit normal distribution.

The results of our DQN experiments, presented in Figure~\ref{grid_plot_NN} reflect some of the expected behaviour, though complicated by a number of subtleties. With 4 hidden units the agent is able to converge to good performance on each counterexample, though convergence was very slow on the worst-case example. With 2 hidden units we see that the oscillating example causes the action-values of the two actions to converge close to the same value, consistently causing policy oscillations and reduced expected return. For the multiple fixed point example many runs of the underparameterized agent initially approach an approximate value function corresponding to the inferior policy. However, in the long-run, every run was able to converge to the superior policy. It is unclear whether this represents a general characteristic of DQN which allows it to eventually escape from such inferior fixed points, or simply a quirk of our particular architecture and problem setting. On the worst-case example, the underparameterized agent usually converged to an evaluation corresponding to the inferior policy, though 7 out of 30 runs instead converged to a point where both actions had similar estimated values, resulting in policy oscillations.
\begin{figure}[t]
\centering
\begin{subfigure}[t]{\columnwidth}
\includegraphics[width=\textwidth]{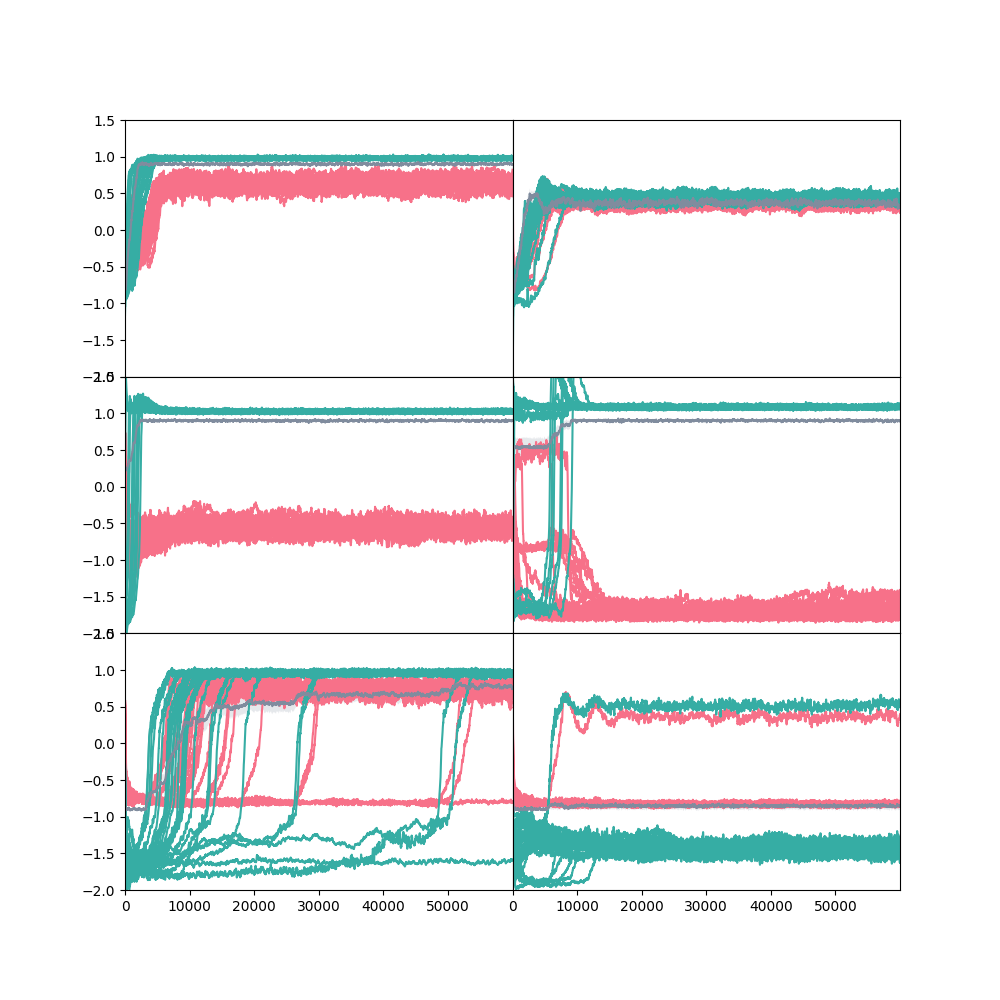}
\end{subfigure}
\caption{Pathological behaviours can arise with neural network function approximation. Columns show results for different numbers of hidden layer units, rows show results on different counterexamples. Each plot shows the evolution of the approximated values of the left and right actions, plotted for each of 30 runs individually. The average return is shown in grey with error bars corresponding to standard error.}
\label{grid_plot_NN}
\end{figure}

\section{Conclusion}
This report highlights a fundamental issue with RL methods that combine on-policy evaluation and greedification. We show for the first time how this issue can cause convergence to the worst possible policy, in addition to policy oscillation and multiple fixed points. By doing so, we hope to help direct the search for potential solutions to this issue.

Of the three possible behaviours highlighted in the section~\ref{counterexamples}, we see the worst-case counterexample as the most problematic. The existence of fixed points can be guaranteed by using a continuous policy class~\citep{de2000existence}. \citet{perkins2003convergent} further introduce a version of API that is guaranteed to converge to a unique fixed point, albeit under fairly restrictive conditions.

Our results show that even when a unique fixed point exists, it can correspond to an arbitrarily bad policy. Even when the optimal policy can be represented and evaluated exactly (in the sense of on-policy evaluation error), it may be repulsive. We consider this a problem with algorithms that combine on-policy approximate policy evaluation and greedification. Such a situation cannot occur in MC-PG, but can occur with both action-value based methods like Q-learning, and AC methods which use an approximate state-value function to estimate the policy-gradient. 

Despite the issues we highlight here, value-based methods can be very powerful as a means to reduce the variance from estimating the policy-gradient from returns alone, and enabling learning immediately from single transitions. For these reasons, it is important that we try to understand the issues with the current standard approaches, and how we might be able to improve upon them.

Methods for mitigating the issues we highlight could take a variety of forms. One possibility is to develop better representation learning methods which remove the kind of bad generalization that leads to pathological behaviors. Another, perhaps orthogonal, approach is to develop algorithms that are more sensitive to the dynamics between the value function and policy. By developing both these directions we can perhaps hope to design algorithms that make the most out of the current representation, while also improving the representation over time to further increase performance.


\bibliography{refs}
\bibliographystyle{icml2020}
\end{document}